\documentclass{article}

\usepackage[final]{ewrl_2026}
\usepackage[utf8]{inputenc}
\usepackage[T1]{fontenc}
\usepackage{hyperref}
\usepackage{url}
\usepackage{booktabs}
\usepackage{amsfonts}
\usepackage{amsmath, amssymb, amsthm}
\usepackage{mathtools}
\usepackage{nicefrac}
\usepackage{microtype}
\usepackage{xcolor}
\usepackage{graphicx}
\usepackage{algorithm}
\usepackage{algpseudocode}
\usepackage{enumitem}

\newtheorem{theorem}{Theorem}

\newtheorem{corollary}{Corollary}
\newtheorem{remark}{Remark}
\theoremstyle{definition}

\DeclareMathOperator*{\argmin}{arg\,min}
\DeclareMathOperator*{\argmax}{arg\,max}

\title{Differentiating Bisimulation Metrics:\\
A Framework for Parametric Markov Chain Fitting\\
via Bicausal Optimal Transport}

\author{Sergio Calo$^{1}$\thanks{Corresponding author: \texttt{sergio.calo@upf.edu}} \,\,\,\, Amy Zhang$^{2}$ \,\,\,\, Javier Segovia-Aguas$^{1}$ \,\,\,\, Anders Jonsson$^{1}$
 \\
 $^{1}$Universitat Pompeu Fabra, Barcelona, Spain\\
 $^{2}$University of Texas at Austin
 }

\begin{document}

\maketitle

\begin{abstract}
Many problems in sequential decision-making, such as imitation learning from
observations, state-space compression, world-model learning, and sim-to-real
transfer, can be reduced to learning a model such that a notion of distance 
with respect to the target process is minimized.  We consider this general framework 
and consider the \emph{bisimulation metric}, equivalently Bicausal Optimal Transport (BOT), as the 
notion of distance to minimize. We show that BOT, since it can be formulated as a linear program (LP),
is differentiable with respect to the model dynamics. We then derive an exact closed-form
gradient via the envelope theorem applied to the LP saddle point.  The result is a general algorithm,
\emph{Differentiable Bicausal Optimal Transport} (\textsc{D-BOT}), that can be applied to each of the problems above.
The proposed algorithm learns the best model
by alternating between distance computation and gradient steps.  We apply
\textsc{D-BOT} for three different settings: state-space compression,
parametric model learning, and imitation learning from observations (ILfO).
We show empirical results that confirm the viability of all three
instantiations.
\end{abstract}

% ── Main sections ────────────────────────────────────────────────────

\section{Introduction}
\label{sec:intro}

Many problems in sequential decision-making reduce to a common template: define a
parametric model of a stochastic process, sample transitions from a reference
process, and the requirement that the two be brought as close as possible by
adjusting the model's parameters.  In Imitation Learning from Observations (ILfO), for instance, one asks for a
policy whose induced chain matches an expert's transitions without access to
the expert's actions.  System identification and sim-to-real transfer require
tuning a set of parameters until the simulator
reproduces real-world trajectories.  State-space compression aims to find a small
abstract process whose dynamics faithfully approximate those of a much larger
original system.

The central difficulty is that ``as close as possible'' must be defined
carefully.  Stochastic processes unfold over time, and a meaningful distance
between them should respect this temporal, causal structure: it should penalise
not just mismatches in where the process spends its time, but mismatches in how
it moves from one state to the next.  A distance based on marginal
state-occupancy, for instance, cannot distinguish two policies that may produce
the same marginal distribution over states while producing very different behaviors.  
What is needed is a distance that
couples the two processes \emph{jointly across time}, so that the cost of
matching one trajectory to another reflects the sequential structure of both.
At the same time, the distance must be practically usable: it must be
estimable from sample transitions alone, without requiring knowledge of the
underlying transition kernels, and it must be differentiable with respect to
the model parameters so that gradient-based optimization can be applied.

We argue that the \emph{bisimulation metric}, recently shown to be equivalent
to the \emph{bicausal optimal transport distance}~\citep{calo2024bisimulation},
which couples two processes in a way that respects the causal, temporal ordering of both chains, satisfies all of these requirements simultaneously.  This connection to
optimal transport opens a rich set
of algorithmic tools.  In particular, \citet{calo2025somcot} recently showed
that the bisimulation metric can be computed from sample transitions alone,
without knowledge of either transition kernel, via a stochastic primal-dual
algorithm called \textsc{SOMCOT}.  The present paper
shows that the same LP formulation that enables \textsc{SOMCOT} also makes the
bisimulation metric \emph{differentiable} with respect to the parameters of
the model being fitted.  The gradient has a closed-form solution, obtained by
applying the envelope theorem \citep{danskin1967theory} to the LP saddle point: it involves only two of
the six dual variables that \textsc{SOMCOT} already computes as a byproduct,
and requires no differentiation through the inner optimization.

We design \textsc{D-BOT} (\emph{Differentiable Bicausal Optimal Transport})
around this gradient: the algorithm alternates between running \textsc{SOMCOT}
and taking a gradient step on the model parameters.
The outer loop is the same regardless of the application; what differs is only
how $\nabla_\theta \nu_\theta$ (that is, the gradient of the transition occupancy) is
computed for each particular parametrization.  For state-space compression and parametric model learning, where $P_\theta$ is
an explicit differentiable kernel, both share an identical gradient computation and are unified in
Section~\ref{sec:representation}.  For imitation
learning from observations, where $P_\theta$ is parametrized only implicitly through a parametric policy acting in
an MDP, the gradient takes the form of a policy gradient under an implicitly
defined reward (Section~\ref{sec:ilfo}).

\paragraph{Contributions.}
We derive an exact closed-form gradient of the bisimulation metric with respect
to any differentiable parametrization of the second chain, via the envelope
theorem applied to the LP saddle point of~\citet{calo2025somcot}
(Section~\ref{sec:framework}).  Based on this we propose \textsc{D-BOT}, a
general algorithm for parametric chain fitting that tackles many of the main
problems in reinforcement learning. In particular, we instantiate it for state-space
compression, parametric model learning, and imitation learning from observations,
providing complete algorithmic descriptions and empirical results
(Sections~\ref{sec:representation}--\ref{sec:ilfo}).

\subsection{Related work.}
Imitation learning from observations has been studied as a distribution-matching problem, with methods minimising either KL divergences~\citep{kostrikov2019imitation} or optimal transport distances~\citep{dadashi2021primal, yan2024pwdice, pham2025iostom} between state-occupancy marginals.
Representation learning and state-space compression have been approached through the lens of bisimulation~\citep{ferns2004metrics, givan2003equivalence, castro2020scalable, zhang2021learning, chen2022learning, kemertas2022approximate}, with recent work focusing on scalable and differentiable approximations of the bisimulation metric.
Parametric model learning and system identification have been treated as separate estimation problems.
These three lines of work have developed largely in isolation. However,
to the best of our knowledge, no prior work provides a unified framework that incorporates all of
them.  We discuss the related work of each area and the connections to this work separately in
Appendix~\ref{app:related_work}.

\section{Background}
\label{sec:prelim}

We consider two stationary Markov processes
$\mathcal{M}_X = (\mathcal{X}, P_X, \nu_{0,X})$ and
$\mathcal{M}_Y = (\mathcal{Y}, P_Y, \nu_{0,Y})$,
where $\mathcal{X}$ and $\mathcal{Y}$ are finite state spaces,
$P_X : \mathcal{X} \to \Delta(\mathcal{X})$ and
$P_Y : \mathcal{Y} \to \Delta(\mathcal{Y})$ are the transition kernels, and
$\nu_{0,X}, \nu_{0,Y}$ are the initial distributions.  We assume without loss
of generality that both initial distributions are Dirac measures on fixed
states $x_0$ and $y_0$.

Given a ground cost $c : \mathcal{X} \times \mathcal{Y} \to \mathbb{R}_+$,
the \emph{discounted total cost} between two trajectories
$\bar{x} = (x_0, x_1, \ldots)$ and $\bar{y} = (y_0, y_1, \ldots)$ is
$c_\gamma(\bar{x}, \bar{y}) = \sum_{t=0}^{\infty} \gamma^t c(x_t, y_t)$
for a discount factor $\gamma \in (0,1)$.

\subsection{Bicausal Couplings}

A coupling of $\mathcal{M}_X$ and $\mathcal{M}_Y$ is a joint process on
$\mathcal{X} \times \mathcal{Y}$ whose marginals equal $\mathcal{M}_X$ and
$\mathcal{M}_Y$ respectively.  A coupling $\pi$ is \emph{bicausal} if, for
all $n \geq 0$,
\[
  \sum_y \pi(xy \mid \bar{x}^{n-1}\bar{y}^{n-1}) = P_X(x \mid \bar{x}^{n-1})
  \quad\text{and}\quad
  \sum_x \pi(xy \mid \bar{x}^{n-1}\bar{y}^{n-1}) = P_Y(y \mid \bar{y}^{n-1}).
\]
Intuitively, bicausal couplings respect the temporal structure of both chains:
neither chain can peek at the other's future.  Let $\Pi_{\mathrm{bc}}$ denote
the set of all bicausal couplings.  \citet{moulos2021bicausal} showed that
restricting to \emph{Markovian} bicausal couplings---where the joint
transition at time $t$ depends only on the current pair $(X_t, Y_t)$---does
not increase the optimal transport cost.  This reduction to Markovian couplings
is what enables the LP formulation below.

\subsection{Bisimulation Metric}

The \emph{bisimulation metric} (equivalently, the bicausal OT distance)
between $\mathcal{M}_X$ and $\mathcal{M}_Y$ is
\begin{equation}
  d_\gamma(\mathcal{M}_X, \mathcal{M}_Y)
  \;=\; \inf_{\pi \in \Pi_{\mathrm{bc}}}
  \int c_\gamma(X, Y)\, d\pi(X, Y).
  \label{eq:distance}
\end{equation}
As noted by \citet{calo2024bisimulation}, this quantity coincides with the bisimulation metric of~\citet{ferns2004metrics}
and~\citet{givan2003equivalence} when the ground cost is the absolute difference in
state rewards:
\begin{equation}
  c(x,y) = |r(x) - r(y)|,
  \label{eq:reward_cost}
\end{equation}
where $r : \mathcal{X} \to \mathbb{R}$ is the reward function.

\paragraph{Connection to reinforcement learning.}
In the RL setting, $\mathcal{M}_X$ and $\mathcal{M}_Y$ are not given
directly; they are typically \emph{induced} by policies acting in an MDP.
Given a discounted MDP $(\mathcal{S}, \mathcal{A}, P, r, \gamma)$ and a policy
$\pi : \mathcal{S} \to \Delta(\mathcal{A})$, the policy induces a Markov chain
over $\mathcal{S}$ with transition kernel
$P^\pi(s'|s) = \sum_a \pi(a|s)\,P(s'|s,a)$.
Notably, $\pi$ also induces a reward function in the Markov chain as $r^\pi : \mathcal{S} \to \mathbb{R}$
defined by $r^\pi(s) = \sum_a \pi(a|s)\,r(s,a)$.
Any Markov chain $\mathcal{M}_X$ can therefore be viewed as arising from some
MDP-policy pair, and the bisimulation metric between two chains corresponds
to comparing the behaviors induced by two policies (or two MDPs) in a
principled, causally-aware manner.

\subsection{LP Formulation}
\label{sec:lp}

The distance (Eq.~\ref{eq:distance}) can be rewritten as a linear program in the
\emph{occupancy coupling}, defined for a coupling $\pi \in \Pi_{\mathrm{bc}}$ as
\[
  \mu^\pi(x,y,x',y')
  = (1-\gamma)\sum_{t=0}^\infty \gamma^t
    \mathbb{P}_\pi\bigl[X_t{=}x,\, Y_t{=}y,\, X_{t+1}{=}x',\, Y_{t+1}{=}y'\bigr],
\]
where $\mathbb{P}_\pi$ is the probability measure on the joint process $(X_t,Y_t)_{t\geq 0}$ induced by the coupling~$\pi$.
Introduce the \emph{marginal transition occupancy measures}:
\begin{align*}
  \nu_X(x,x') &= (1-\gamma)\sum_{t \geq 0} \gamma^t
    \mathbb{P}[X_t{=}x,\, X_{t+1}{=}x'], \\
  \nu_Y(y,y') &= (1-\gamma)\sum_{t \geq 0} \gamma^t
    \mathbb{P}[Y_t{=}y,\, Y_{t+1}{=}y'].
\end{align*}
\citet{calo2025somcot} show that for all $x, y, x', y'$, the distance
$d_\gamma(\mathcal{M}_X,\mathcal{M}_Y) = \inf_{\mu,\lambda_X,\lambda_Y} \langle \mu, c \rangle$
subject to:
\begin{align}
  \sum_{x',y'} \mu(x,y,x',y')
    &= \gamma \sum_{\hat{x},\hat{y}} \mu(\hat{x},\hat{y},x,y)
    + (1-\gamma)\nu_0(x,y),
  \tag{flow} \label{eq:flow}\\
  \sum_{y'} \mu(x,y,x',y')
    &= \nu_X(x,x')\,\lambda_X(y|x),
  \tag{causal-$X$} \label{eq:cx}\\
  \sum_{x'} \mu(x,y,x',y')
    &= \nu_Y(y,y')\,\lambda_Y(x|y),
  \tag{causal-$Y$} \label{eq:cy}
  \notag
\end{align}
for some $\lambda_X \in \Delta(\mathcal{Y})^{\mathcal{X}}$ and
$\lambda_Y \in \Delta(\mathcal{X})^{\mathcal{Y}}$.
Crucially, the constraints involve only the \emph{occupancy measures} $\nu_X$
and $\nu_Y$, not the transition kernels directly, making them computable from sample transitions alone.

\paragraph{Interpretation of $\lambda_X$ and $\lambda_Y$.}
From (Eq.~\ref{eq:cx}), $\lambda_X(y|x)$ is the conditional distribution of $Y$
given $X\!=\!x$ under the optimal coupling. From the lens of representation learning, 
$\lambda$ is interpretable as a \emph{soft encoder} from
observations to abstract states.  Symmetrically, $\lambda_Y(x|y)$ is a
\emph{soft decoder} from abstract states back to observations.

\subsection{Lagrangian and the SOMCOT Algorithm}

We associate dual variables with each constraint: $V \in \mathbb{R}^{\mathcal{X}
\times \mathcal{Y}}$ for the flow constraint (Eq.~\ref{eq:flow}), $\alpha_X \in
\mathbb{R}^{\mathcal{X} \times \mathcal{X} \times \mathcal{Y}}$ for the
causal-$X$ constraint (Eq.~\ref{eq:cx}), and $\alpha_Y \in \mathbb{R}^{\mathcal{X}
\times \mathcal{Y} \times \mathcal{Y}}$ for the causal-$Y$
constraint (Eq.~\ref{eq:cy}).  Writing $\langle \cdot,\cdot\rangle$ for the
Euclidean inner product on the appropriate index set, and defining
\[
  \delta(x,y,x',y') \;\coloneqq\; c(x,y) + \alpha_X(x,x',y) + \alpha_Y(x,y,y')
                          + \gamma V(x',y') - V(x,y),
\]
the Lagrangian compacts to
\begin{equation}
  \mathcal{L}
  \;=\; \langle \mu, \delta \rangle
      \,-\, \langle \nu_X \lambda_X,\, \alpha_X \rangle
      \,-\, \langle \nu_Y \lambda_Y,\, \alpha_Y \rangle
      \,+\, (1-\gamma)\,\langle \nu_0,\, V \rangle,
  \label{eq:lagrangian}
\end{equation}
where $(\nu_X\lambda_X)(x,x',y) \coloneqq \nu_X(x,x')\,\lambda_X(y|x)$ and
similarly for $\nu_Y\lambda_Y$.  The distance equals the saddle-point value:
\begin{equation}
  d_\gamma(\mathcal{M}_X,\mathcal{M}_Y)
  \;=\; \min_{\mu,\lambda_X,\lambda_Y}\; \max_{\alpha_X,\alpha_Y,V}\;
  \mathcal{L}(\mu,\lambda_X,\lambda_Y;\alpha_X,\alpha_Y,V).
  \label{eq:saddle}
\end{equation}

\textsc{SOMCOT}~\citep{calo2025somcot} solves (Eq.~\ref{eq:saddle}) from sample
transitions alone, without knowledge of $P_X$ or $P_Y$.  It draws one
transition $(X_k, X'_k) \sim \nu_X$ and $(Y_k, Y'_k) \sim \nu_Y$ per
iteration, updates primal variables $(\mu, \lambda_X, \lambda_Y)$ via
stochastic mirror descent with entropic regularization, and dual variables
$(\alpha_X, \alpha_Y, V)$ via projected gradient ascent.  Upon termination, \textsc{SOMCOT} returns the
time-averaged iterates
$(\bar{\mu}, \bar{\lambda}_X, \bar{\lambda}_Y, \bar{\alpha}_X,
\bar{\alpha}_Y, \bar{V})$
along with the distance estimate $\hat{d}_\gamma = \langle \bar\mu, c\rangle$.

\subsection{Problem Formulation}
\label{sec:problem}

We now turn to the setting that motivates this work.  Suppose the first chain
$\mathcal{M}_X$ is a fixed \emph{target} process, representing, for instance,
an expert's behavior or a reference environment.  The second chain is
parametrized: $\mathcal{M}_Y = \mathcal{M}_\theta$, where $\theta \in
\Theta \subseteq \mathbb{R}^d$ governs the transition kernel $P_\theta :
\mathcal{Y} \to \Delta(\mathcal{Y})$ and hence the occupancy measure
$\nu_\theta$.  We seek the parameter vector that brings $\mathcal{M}_\theta$
as close as possible to $\mathcal{M}_X$ under the bisimulation metric.
The mapping $\theta \mapsto \nu_\theta$ refers to the transition occupancy measure
of $\mathcal{M}_\theta$ (Eq.~\ref{eq:nu_theta}), which is determined by $P_\theta$ and the initial distribution $\nu_{0,Y}$.
Note that $\nu_\theta$ satisfies the causal-$Y$ constraint (Eq.~\ref{eq:cy}), and any
parametrization $\theta$ that determines $P_\theta$ therefore also determines $\mu$ through that constraint:
\begin{equation}
  \theta^* \;=\; \argmin_{\theta \in \Theta}\;
  d_\gamma\!\left(\mathcal{M}_X,\, \mathcal{M}_\theta\right).
  \label{eq:problem}
\end{equation}
Because $d_\gamma(\mathcal{M}_X, \mathcal{M}_\theta)$ is the value of the LP
in Section~\ref{sec:lp}, which is linear, and hence convex, in $\nu_\theta$,
the objective in~(Eq.~\ref{eq:problem}) inherits this property whenever
$\theta \mapsto \nu_\theta$ is itself convex.  More generally, even when
this map is nonlinear (for example, when $P_\theta$ is induced by a neural-network
policy), first-order methods remain applicable provided the gradient
$\nabla_\theta d_\gamma(\mathcal{M}_X, \mathcal{M}_\theta)$ can be computed
or estimated efficiently.

Solving~(Eq.~\ref{eq:problem}) requires differentiating through the inner
optimal-transport problem, which couples $\mathcal{M}_X$ and
$\mathcal{M}_\theta$ via the bicausal LP.  The next section derives a closed
form for this gradient and builds the \textsc{D-BOT} algorithm around it.

\section{The \textsc{D-BOT} Framework}
\label{sec:framework}

Consider the setting introduced in (Eq.~\ref{eq:problem}); the missing ingredient is $\nabla_\theta d_\gamma(\mathcal{M}_X, \mathcal{M}_\theta)$.
In this section we show how to obtain this gradient and describe the resulting algorithm.

\subsection{Gradient of the Bisimulation Metric}

Examining the Lagrangian (Eq.~\ref{eq:lagrangian}), $\theta$ enters \emph{only}
through the term involving $\nu_\theta$, namely:
\begin{equation}
  \mathcal{L}_\theta(\mu,\lambda_X,\lambda_Y;\alpha_X,\alpha_Y,V)
  = \underbrace{\langle\mu,c\rangle + \cdots}_{\text{independent of }\theta}
  \;-\; \sum_{x,y,y'} \nu_\theta(y,y')\,\lambda_Y(x|y)\,\alpha_Y(x,y,y').
  \label{eq:lag_theta}
\end{equation}

\begin{theorem}[Gradient of the bisimulation metric]
  \label{thm:gradient}
  Let $\theta \mapsto \nu_\theta$ be differentiable, and let $d_\gamma$ be the bisimulation metric.
  At the saddle point $(\mu^*, \lambda_X^*, \lambda_Y^*; \alpha_X^*, \alpha_Y^*, V^*)$,
  \begin{equation}
    \nabla_\theta\, d_\gamma(\mathcal{M}_X, \mathcal{M}_\theta)
    \;=\; -\sum_{x,y,y'} \lambda^*_Y(x|y)\;\alpha^*_Y(x,y,y')\;
    \nabla_\theta\, \nu_\theta(y,y').
    \label{eq:gradient}
  \end{equation}
\end{theorem}

\begin{proof}[Proof sketch]
  Since $\theta$ appears only through $\nu_\theta$ in~\eqref{eq:lag_theta},
  and the primal variables $(\mu^*, \lambda_X^*, \lambda_Y^*)$ satisfy all LP
  constraints~\eqref{eq:flow}--\eqref{eq:cy} at the saddle point, the envelope
  theorem for parametric LPs~\citep{puterman1994markov} states that the
  derivative of the optimal value with respect to $\theta$ equals the partial
  derivative of the Lagrangian at the saddle point.  The expression
  in~\eqref{eq:gradient} follows immediately by differentiating the
  $\nu_\theta$-dependent term in~\eqref{eq:lag_theta}.  A complete proof is
  given in Appendix~\ref{app:proof}.
\end{proof}

\begin{remark}[Consecuences of the envelope theorem]
  Due to the envelope theorem, the variables $\lambda^*_Y$ and $\alpha^*_Y$ are treated as constants with
  respect to $\theta$ at the saddle point.  Only $\nu_\theta$ needs to be
  differentiated, so the gradient computation is cheap regardless of the complexity of the \textsc{SOMCOT} inner loop.
\end{remark}

\begin{remark}[Convexity in $\nu_\theta$]
  Because $d_\gamma(M_X,M_\theta)$ is the value of a linear program, it is convex as a function of the transition occupancy measure $\nu_\theta$. However, the parametrization $\theta \mapsto \nu_\theta$ induced by the Markov dynamics is generally nonlinear, and may be highly nonconvex (for example, when $P_\theta$ is represented by a neural network policy). Consequently, the outer optimization problem in (Eq.~\ref{eq:problem}) is, in general, a nonconvex bilevel optimization problem. The contribution of this work is therefore not a global convexity result, but rather the derivation of an exact first-order gradient of the bisimulation metric with respect to the model parameters.

\end{remark}

\subsection{The \textsc{D-BOT} General Algorithm}

Theorem~\ref{thm:gradient} translates directly into a first-order optimization 
algorithm for minimizing the bisimulation distance $d_\gamma(M_X,M_\theta)$ with respect to $\theta$:
alternate between running \textsc{SOMCOT} to produce dual certificates and
taking a gradient step on $\theta$.  We call this \textsc{D-BOT}
(\emph{Differentiable Bicausal Optimal Transport}) and state it as
Algorithm~\ref{alg:dbot}.

\begin{algorithm}[h]
  \caption{\textsc{D-BOT}: Differentiable Bicausal Optimal Transport Fitting}
  \label{alg:dbot}
  \begin{algorithmic}[1]
    \Require Sample access to $\mathcal{M}_X$; parametric chain $\mathcal{M}_\theta$
             (any differentiable parametrization); cost $c$; step size
             $\eta$; iterations $K$.
    \State Initialize $\theta_0$ (e.g.\ uniformly at random or from a prior).
    \For{$k = 1, 2, \ldots, K$}
      \State $(\bar\lambda_{Y,k},\, \bar\alpha_{Y,k}) \leftarrow \textsc{SOMCOT}(\mathcal{M}_X,\, \mathcal{M}_{\theta_{k-1}},\, c)$ \hfill\textit{// OT step}
      \State \textbf{Gradient computation:}
      \State Compute $\widehat{\nabla_\theta d_\gamma}$ via (Eq.~\ref{eq:gradient}), differentiating $\nu_{\theta_{k-1}}$ for the specific parametrization (see Sections~\ref{sec:repr_setting}~and~\ref{sec:ilfo}).
      \State $\theta_k \leftarrow \theta_{k-1} - \eta\, \widehat{\nabla_\theta d_\gamma}$ \hfill\textit{// parameter update}
    \EndFor \\
    \Return $\theta_{K}$; encoder $\bar{\lambda}_X$;
            decoder $\bar{\lambda}_Y$.
  \end{algorithmic}
\end{algorithm}

The outer loop of Algorithm~\ref{alg:dbot} does not change between
applications; what varies is only how $\nabla_\theta \nu_\theta$ is computed:
\begin{itemize}
  \item \textbf{Representation learning} (Section~\ref{sec:representation}):
    $P_\theta$ is an explicit differentiable kernel, so $\nu_\theta$ is
    computable in closed form and its gradient is obtained by pathwise
    autodifferentiation through a linear system solve.  This covers both
    parametric model learning ($|\mathcal{Y}| = |\mathcal{X}|$, system
    identification) and state-space compression ($|\mathcal{Y}| \ll
    |\mathcal{X}|$, dimensionality reduction); the gradient machinery is
    identical in both cases.
  \item \textbf{Imitation learning from observations} (Section~\ref{sec:ilfo}):
    $P_\theta$ is induced by a policy $\pi_\theta$ acting in an MDP and is
    not directly differentiable.  The gradient of $\nu_\theta$ is estimated
    via the policy gradient theorem.
\end{itemize}

% =========================================================================
\section{Representation Learning via \textsc{D-BOT}}
\label{sec:representation}
% =========================================================================

In this section we introduce the application of the general algorithm \textsc{D-BOT}
for the settings of model learning and state-space compression.

\subsection{Setting}
\label{sec:repr_setting}

Given sample transitions from a target chain $\mathcal{M}_X = (\mathcal{X},
P_X, \nu_{0,X})$, we want to fit a parametric chain
$\mathcal{M}_\theta = (\mathcal{Y}, P_\theta, \nu_{0,Y})$ by minimizing
\begin{equation}
  \theta^* \;=\; \argmin_{\theta}\;
  d_\gamma\!\left(\mathcal{M}_X,\, \mathcal{M}_\theta\right).
  \label{eq:repr_obj}
\end{equation}

Here $\theta$ is any parameter vector that determines the row-stochastic
kernel $P_\theta : \mathcal{Y} \to \Delta(\mathcal{Y})$. We will consider the following two scenarios:

\begin{itemize}
  \item \textbf{Parametric model learning} ($|\mathcal{Y}| = |\mathcal{X}|$):
    $\mathcal{M}_\theta$ lives on the same state space as $\mathcal{M}_X$
    and is fitted to reproduce its dynamics as faithfully as possible.
    This is a system-identification task; the bisimulation metric acts as
    the loss.
  \item \textbf{State-space compression} ($|\mathcal{Y}| \ll |\mathcal{X}|$):
    $\mathcal{M}_\theta$ lives on a \emph{smaller} abstract space.
    Minimizing the bisimulation distance simultaneously learns compressed
    dynamics and a soft encoder/decoder pair that relates abstract states
    to the original ones.
\end{itemize}

Beyond the compressed kernel $P_\theta^*$, the optimal coupling $\mu^*$
simultaneously identifies a \emph{soft encoder}
$\lambda_X^*(\cdot|x) \in \Delta(\mathcal{Y})$ and a \emph{soft decoder}
$\lambda_Y^*(\cdot|y) \in \Delta(\mathcal{X})$ such that the latent
dynamics of $\mathcal{M}_X$ factor through $\mathcal{M}_{\theta^*}$ with
minimal distortion. If $d_\gamma(\mathcal{M}_X, \mathcal{M}_{\theta^*}) = 0$,
the hard encoder $\hat\phi(x) = \argmax_y \lambda_X^*(y|x)$ realizes an
exact aggregation~\citep{givan2003equivalence}.

\subsection{Computing $\nabla_\theta \nu_\theta$}

For an explicit kernel parametrization $\theta \mapsto P_\theta$, the
transition occupancy $\nu_\theta$ factors as
$\nu_\theta(y, y') = \rho_\theta(y)\, P_\theta(y'|y)$,
where the discounted state-occupancy $\rho_\theta \in \mathbb{R}^{n_Y}$ (where $n_Y = |\mathcal{Y}|$)
is the unique solution of
\begin{equation}
  (I - \gamma P_\theta^\top)\,\rho_\theta \;=\; (1-\gamma)\,\nu_{0,Y},
  \label{eq:rho_system}
\end{equation}
giving the closed-form expression
\begin{equation}
  \nu_\theta(y, y')
  \;=\;
  \Bigl[(1-\gamma)(I - \gamma P_\theta^\top)^{-1}\nu_{0,Y}\Bigr]_y
  \cdot P_\theta(y'|y).
  \label{eq:nu_theta}
\end{equation}
Substituting into (Eq.~\ref{eq:gradient}) and detaching the dual certificates
$(\bar\lambda_Y, \bar\alpha_Y)$ returned by \textsc{SOMCOT},
\begin{equation}
  \nabla_\theta\, d_\gamma(\mathcal{M}_X, \mathcal{M}_\theta)
  \;=\; -\sum_{x,y,y'} \bar\lambda_Y(x|y)\;\bar\alpha_Y(x,y,y')\;
  \nabla_\theta\, \nu_\theta(y,y').
  \label{eq:grad_repr}
\end{equation}
Provided $P_\theta$ is differentiable in $\theta$, the chain rule applies
directly: $\nabla_\theta \nu_\theta$ is obtained by differentiating
through (Eq.~\ref{eq:nu_theta}), where the backward pass through the linear solver is handled implicitly by automatic differentiation.%
\footnote{Any standard autodiff framework (e.g.\ JAX, PyTorch) handles the backward pass through the linear system solve automatically.}

The full pseudocode for this procedure (\textsc{D-BOT-Repr}) is given as Algorithm~\ref{alg:compress} in Appendix~\ref{app:algorithms}.  For model learning set $n_Y = n_X$ and $\nu_{0,Y} = \nu_{0,X}$; for
compression choose $n_Y \ll n_X$ and set $\nu_{0,Y}$ to any fixed
distribution (e.g.\ uniform).  In both cases the encoder $\bar\lambda_X$ and
decoder $\bar\lambda_Y$ are returned by \textsc{SOMCOT} at no extra cost; for
model learning they are not used.

% -------------------------------------------------------------------------
\subsection{Experimental results}
\label{sec:repr_exp}
% -------------------------------------------------------------------------
We test Algorithm~\ref{alg:compress} for both settings, comparing against the exact distance from Sinkhorn Policy Iteration (SPI)~\cite{calo2024bisimulation}, which requires full kernel knowledge.

\paragraph{Model Learning:}
We fit a parametric model on a \emph{discrete random walk with drift}: $N=16$ states with reflecting boundaries, drift parameter $f$ (true value $f_{\mathrm{true}}=1.5$, $\gamma=0.95$). We initialize with $f_0 = -2$, strongly biased in the wrong direction..
Figure~\ref{fig:drift_results} shows gradient updates monotonically steering $f$ toward $f_{\mathrm{true}}$, with a brief overshoot due to finite step size.
The learned drift slightly overshoots beyond $f_{\mathrm{true}} = 1.5$ due to small approximation errors in the SOMCOT distance estimates.
The experiment confirms that the bisimulation gradient correctly identifies the generating parameter from sample transitions alone, even from a heavily misspecified initialization.

\begin{figure}[h!]
  \centering
  \includegraphics[scale=0.5]{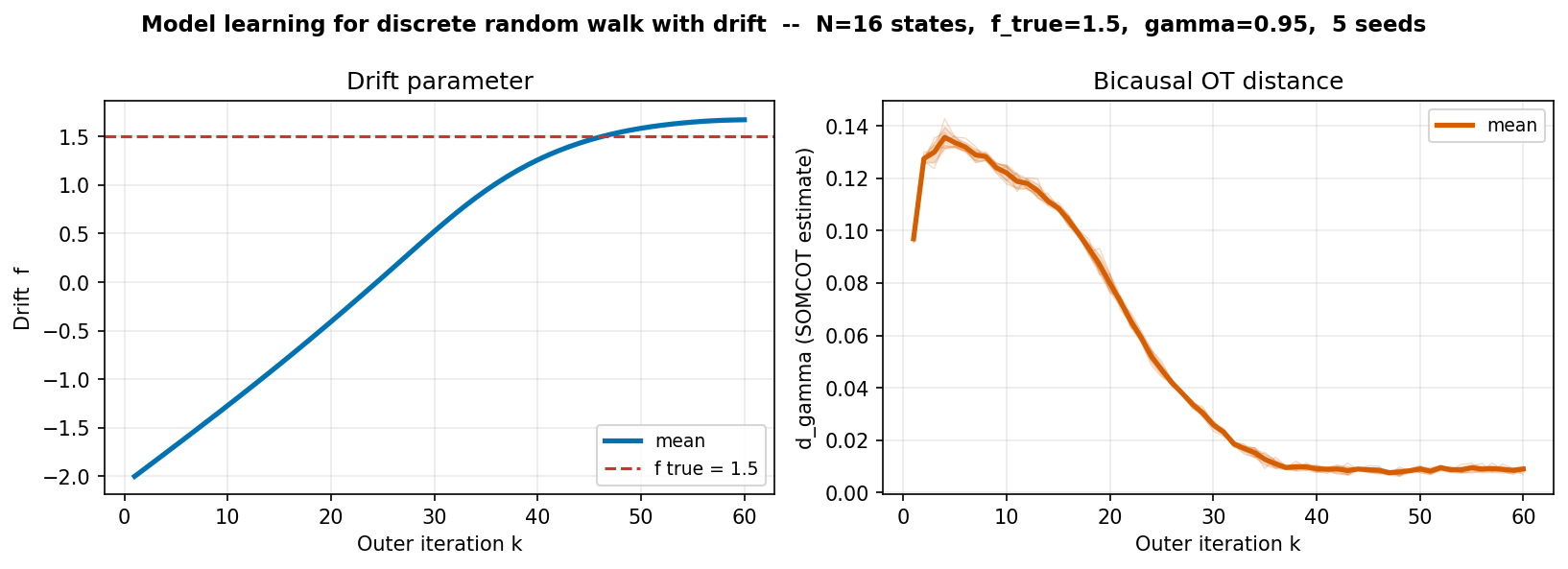}
  \caption{Model learning ($n_Y = n_X = 16$, $f_{\mathrm{true}}=1.5$,
    $\gamma=0.95$).
    \emph{Left}: learned drift $f$ versus outer iteration $k$;
    dashed red line marks the true value.
    \emph{Right}: SOMCOT estimate of $d_\gamma(\mathcal{M}_X, \mathcal{M}_\theta)$.}
  \label{fig:drift_results}
\end{figure}

\paragraph{State-Space Compression:}
We compress a \emph{block-chain} of $n_X=25$ states (5 blocks of behaviorally identical states) to $n_Y=5$ abstract states.
The ground cost is the absolute reward difference between $x$ and $y$, where each block shares a reward proportional to its index; This cost is zero precisely when $x$ belongs to the block corresponding to
$y$, and grows linearly with the number of blocks separating $x$ from $y$,
imposing an ordinal geometry on the abstract space.
Figure~\ref{fig:compress_main} shows the bisimulation distance converging to near zero, and the recovered encoder (Figure~\ref{fig:ablation_encoder}) cleanly assigns each group of five states to a single abstract state with no explicit clustering objective.
Overestimating $n_Y$ does not degrade performance (redundant states are left unused), though it increases cost since \textsc{SOMCOT} scales as $\mathcal{O}(n_X^2 n_Y^2)$; see Appendix~\ref{app:ablations} for full ablations.

\begin{figure}[h!]
  \centering
  \includegraphics[scale=0.5]{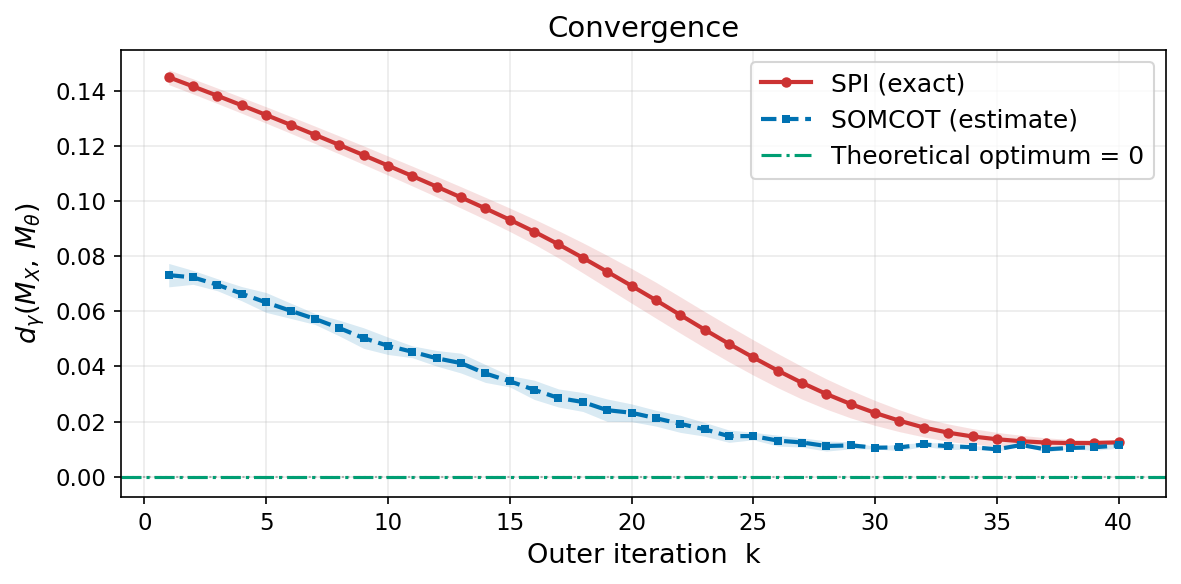}
  \caption{State-space compression ($n_X = 25$, $n_Y = 5$, $\gamma = 0.95$):
    bisimulation distance $d_\gamma(\mathcal{M}_X, \mathcal{M}_\theta)$ versus
    outer iterations.  The 5-state abstract chain faithfully represents the
    original 25-state chain.}
  \label{fig:compress_main}
\end{figure}

\section{Imitation Learning from Observations}
\label{sec:ilfo}

This section instantiates \textsc{D-BOT} for the Imitation Learning from Observations (ILfO) setting, where the learner observes only state transitions from an expert and must recover a policy that reproduces the same induced Markov chain.

\subsection{Setting}

We consider a standard discounted MDP $\mathcal{M} = (\mathcal{S}, \mathcal{A}, P, \gamma)$ as in Section~\ref{sec:prelim}, with an \emph{unknown} reward function.  The learner has access to an expert dataset $\mathcal{D}_E = \{(s_i, s'_i)\}_{i=1}^{N}$ of consecutive state transitions with \emph{no action or reward labels} — strictly harder than standard imitation learning.

Each policy $\pi_\theta$ induces a Markov chain $\mathcal{M}_{\pi_\theta}$ over $\mathcal{S}$ with transition kernel $P_\theta(s'|s) = \sum_{a} \pi_\theta(a|s)\,P(s'|s,a)$.  Let $\mathcal{M}_E$ be the analogous chain induced by the expert.  The objective is
\begin{equation*}
  \min_\theta \; d_\gamma\!\left(\mathcal{M}_E,\, \mathcal{M}_{\pi_\theta}\right).
\end{equation*}

\subsection{Computing $\nabla_\theta \nu_\theta$ via Policy Gradient}

In the ILfO setting, $P_\theta$ is the marginalization of the MDP kernel over
$\pi_\theta$ and is not directly differentiable with respect to $\theta$.  We
instead rewrite the $\nu_\theta$-dependent term in (Eq.~\ref{eq:gradient})
as an expectation under the induced policy.  Using the shared state space
$\mathcal{X} = \mathcal{Y} = \mathcal{S}$:
\begin{equation}
  \sum_{s,s'} \nu_\theta(s,s')
    \sum_a \bar\lambda_Y(a|s)\,\bar\alpha_Y(a,s,s')
  = \mathbb{E}_{\pi_\theta}\!\left[
    \sum_{t=0}^{\infty} \gamma^t
    \sum_a \bar\lambda_Y(a|S_t)\,\bar\alpha_Y(a,S_t,S_{t+1})
  \right].
  \label{eq:rl_objective}
\end{equation}
This is a standard discounted RL objective under the \emph{implicit reward}
\begin{equation}
  r(s,s') \;:=\; -\sum_a \bar\lambda_Y(a|s)\,\bar\alpha_Y(a,s,s').
  \label{eq:implicit_reward}
\end{equation}
\begin{corollary}[Policy gradient for ILfO]
\label{cor:pg_ilfo}
As a corollary of the policy gradient theorem~\citep{sutton1999policy},
\begin{equation}
  \widehat{\nabla_\theta d_\gamma}
  = -\,\mathbb{E}_{(s,s') \sim \nu_{\pi_\theta}}\!\left[
    \nabla_\theta \log\pi_\theta(s'|s)\;Q^r_{\pi_\theta}(s,s')
  \right],
  \label{eq:policy_gradient}
\end{equation}
where $Q^r_{\pi_\theta}$ is the action-value function under the implicit
reward~$r$.
\end{corollary}

Unlike the compression case, both inner and outer optimization loops must run until convergence, as the implicit reward changes each time $\pi_\theta$ is updated: \textsc{SOMCOT} runs to convergence first, producing a stationary $r_k$; only then does the policy update loop converge under that fixed reward.  The full pseudocode is Algorithm~\ref{alg:ilfo} in Appendix~\ref{app:algorithms}.

\paragraph{Interpretation of the implicit reward.}
The reward $r_k(s,s')$ scores each state transition according to how well it
explains the expert's causal transition structure.  Unlike occupancy-based
methods that assign rewards to states alone, $r_k$ is sensitive to
\emph{how} the agent moves between states.  The reward is re-estimated at
every outer iteration as $\pi_\theta$ improves, so it adapts automatically
as the imitating policy approaches the expert.

\subsection{Experimental Results}
\label{sec:ilfo_exp}

\paragraph{Environment and baselines.}
We evaluate \textsc{D-BOT-ILfO} on a 5-state discrete chain MDP with four
actions (left, right, stay, jump-to-start) and stochastic transitions.  The
expert policy is an $\varepsilon$-soft "go right" policy; we consider stochastic 
and deterministic expert policies.  The learner observes only
consecutive state pairs $(s, s')$ from expert rollouts and has no access to
action labels.  We compare against two baselines:
IOSTOM~\citep{pham2025iostom}, an offline ILfO method that
    matches joint state-transition occupancies via Q-learning with LSIQ-style
    targets and advantage-weighted regression for policy extraction;
PW-DICE~\citep{yan2024pwdice}, a one-shot convex program that
    minimises a regularised primal Wasserstein distance between learner and
    expert estimated state occupancies. The cost function used assigns a cost 0 
    if the states are the same, and cost 1 otherwise.

\paragraph{Results.}
Figures~\ref{fig:ilfo_stochastic} and~\ref{fig:ilfo_deterministic} show the
SPI distance and KL divergence over outer iterations for both expert variants.
\textsc{D-BOT-ILfO} monotonically reduces the bisimulation distance across all
settings, and simultaneously drives down the policy KL.
In particular, our method outperform others in the precense of stochasticity in the
expert poilcy. These results confirm that \emph{D-BOT} correctly captures causal transition structure that
marginal occupancy measures cannot distinguish.  

\begin{figure}[h]
  \centering
  \includegraphics[width=\textwidth]{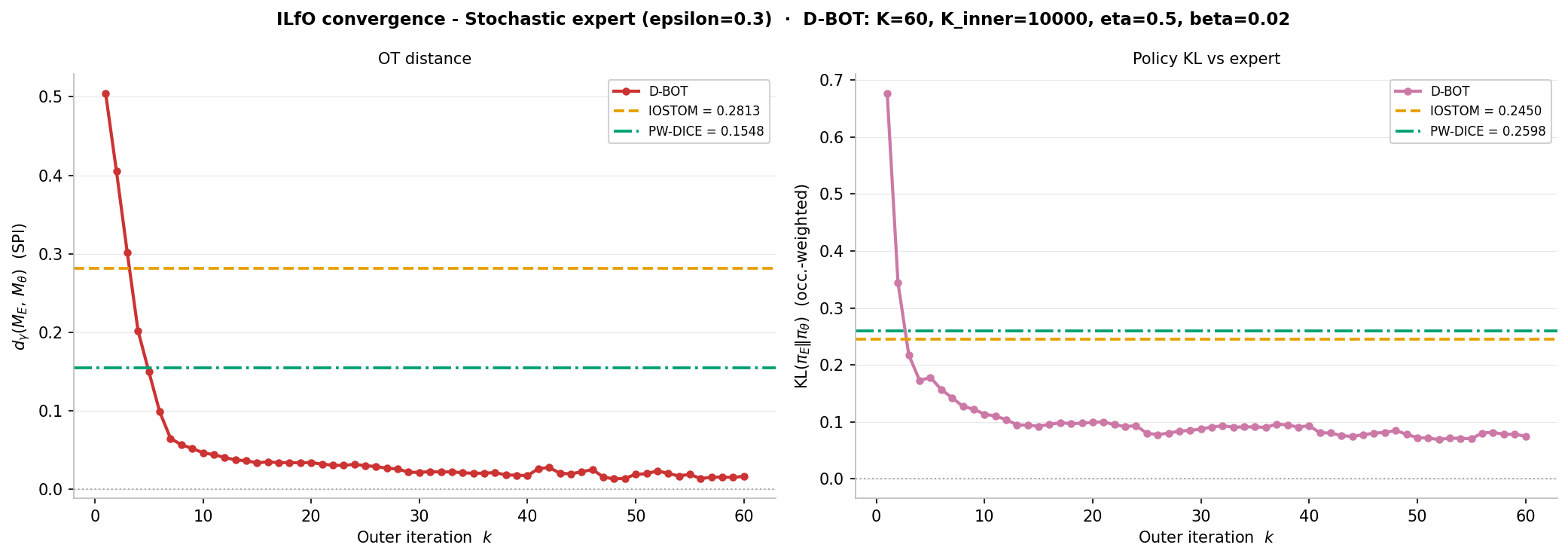}
  \caption{ILfO results with a \textbf{stochastic} expert policy.}
  \label{fig:ilfo_stochastic}
\end{figure}

\begin{figure}[h]
  \centering
  \includegraphics[width=\textwidth]{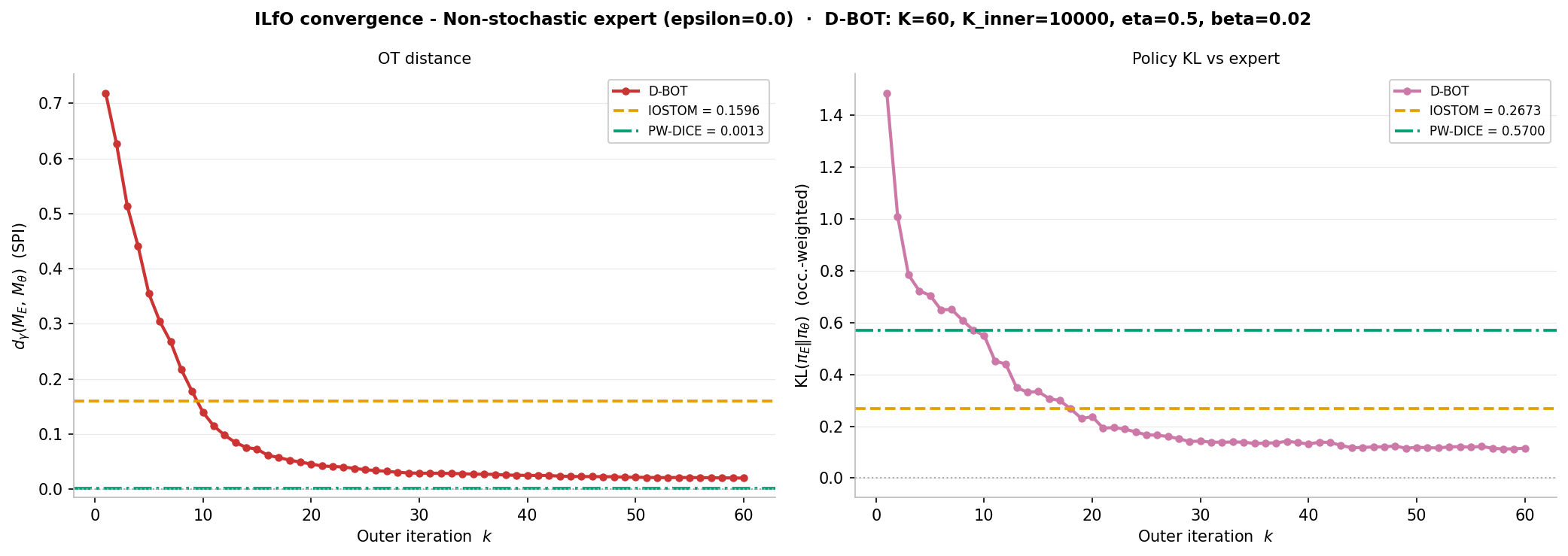}
  \caption{ILfO results with a \textbf{deterministic} expert policy.}
  \label{fig:ilfo_deterministic}
\end{figure}

\section{Conclusion}
\label{sec:conclusion}

We have shown that the bisimulation metric between Markov chains is
differentiable with respect to the parameters of either chain, and that its
gradient admits a clean closed-form expression via the envelope theorem applied to the LP
saddle point of~\citet{calo2025somcot}.  The resulting algorithm,
\textsc{D-BOT}, fits parametric Markov chains to reference processes by
gradient descent on the bicausal OT distance and instantiates naturally for
state-space compression, parametric model learning, and imitation learning
from observations.

\paragraph{Limitations.}
The envelope theorem holds exactly 
only at the true saddle point, so stopping \textsc{SOMCOT} early biases the 
outer gradient proportionally to the inner approximation error. More inner 
iterations reduce this bias but increase computation per outer step. A 
single-loop variant that updates $\theta$ and the primal-dual variables 
jointly would remove this trade-off entirely; the linear dependence of the 
objective on $\nu_\theta$ suggests this is feasible, in the spirit 
of~\citet{ballu2020stochastic} for static OT.

As discussed, the considered optimization problem is generally nonconvex. Therefore D-BOT inherits the 
standard limitations of gradient-based optimization: convergence guarantees are 
local and would depend on initialization, step sizes, and optimization dynamics. Under 
exact inner solves, the method performs gradient descent on the true bisimulation 
metric; however, global optimality cannot in general be guaranteed.

Memory and per-iteration cost scale as 
$\mathcal{O}(|\mathcal{X}|^2|\mathcal{Y}|^2)$, since \textsc{SOMCOT} 
maintains an explicit occupancy coupling, limiting the current implementation 
to chains with at most a few hundred states. Scaling up to large or even continuous state spaces would require 
approximating the primal and dual variables by parametrized functions (for example, neural networks),
but deriving stable stochastic updates for this setting 
is challenging.

% ── Acknowledgements (hidden during blind review by ewrl_2026.sty) ───
\begin{ack}
Amy Zhang is supported by NSF 2340651, NSF 2402650, NSF AI Institute for Foundations of Machine Learning (IFML), TRI, and ARO W911NF-24-1-0193.
Anders Jonsson is partially supported by Spanish grants
PID2023-147145NB-I00 and CEX2021-001195-M, funded
by MCIN/AEI/10.13039/501100011033. Javier Segovia-Aguas is supported by the Ramón y Cajal program, RYC2024-050163-I, funded by MICIU/AEI/10.13039/501100011033 and FSE+.
\end{ack}

\newpage
\bibliographystyle{abbrvnat}
\bibliography{references}

@inproceedings{ballu2020stochastic,
  title={Stochastic optimization for regularized {W}asserstein estimators},
  author={Ballu, Marin and Berthet, Quentin and Bach, Francis},
  booktitle={International Conference on Machine Learning},
  year={2020}
}

@inproceedings{calo2024bisimulation,
  title={Bisimulation metrics are optimal transport distances, and can be computed efficiently},
  author={Calo, Sergio and Jonsson, Anders and Neu, Gergely and Schwartz, Ludovic and Segovia-Aguas, Javier},
  booktitle={Advances in Neural Information Processing Systems},
  year={2024}
}

@article{calo2025somcot,
  title={Distances for {M}arkov chains from sample streams},
  author={Calo, Sergio and Jonsson, Anders and Neu, Gergely and Schwartz, Ludovic and Segovia-Aguas, Javier},
  journal={arXiv preprint arXiv:2505.18005},
  year={2025}
}

@inproceedings{castro2020scalable,
  title={Scalable methods for computing state similarity in deterministic {M}arkov decision processes},
  author={Castro, Pablo Samuel},
  booktitle={AAAI Conference on Artificial Intelligence},
  year={2020}
}

@inproceedings{chang2024oops,
  title={Imitation learning from observation through optimal transport},
  author={Chang, Wei-Di and Fujimoto, Scott and Meger, David and Dudek, Gregory},
  booktitle={Reinforcement Learning Conference},
  year={2024}
}

@inproceedings{chen2022learning,
  title={Learning representations via a robust behavioral metric for deep reinforcement learning},
  author={Chen, Jianda and Pan, Sinno Jialin},
  booktitle={Advances in Neural Information Processing Systems},
  year={2022}
}

@inproceedings{dadashi2021primal,
  title={Primal {W}asserstein imitation learning},
  author={Dadashi, Robert and Hussenot, L{\'e}onard and Geist, Matthieu and Pietquin, Olivier},
  booktitle={International Conference on Learning Representations},
  year={2021}
}

@inproceedings{desharnais1999metrics,
  title={Metrics for labeled {M}arkov systems},
  author={Desharnais, Josee and Gupta, Vineet and Jagadeesan, Radha and Panangaden, Prakash},
  booktitle={International Conference on Concurrency Theory},
  year={1999}
}

@inproceedings{ferns2004metrics,
  title={Metrics for finite {M}arkov decision processes},
  author={Ferns, Norm and Panangaden, Prakash and Precup, Doina},
  booktitle={Uncertainty in Artificial Intelligence},
  year={2004}
}

@article{givan2003equivalence,
  title={Equivalence notions and model minimization in {M}arkov decision processes},
  author={Givan, Robert and Dean, Thomas and Greig, Matthew},
  journal={Artificial Intelligence},
  volume={147},
  number={1-2},
  pages={163--223},
  year={2003}
}

@article{kemertas2022approximate,
  title={Approximate policy iteration with bisimulation metrics},
  author={Kemertas, Mete and Jepson, Allan},
  journal={Transactions on Machine Learning Research},
  year={2022}
}

@inproceedings{kim2022lobsdice,
  title={{LobsDICE}: Offline imitation learning from observations via stationary distribution correction estimation},
  author={Kim, Siyuan and Park, Jayden and Oh, Songhwai},
  booktitle={Advances in Neural Information Processing Systems},
  year={2022}
}

@inproceedings{kostrikov2019imitation,
  title={Imitation learning via off-policy distribution matching},
  author={Kostrikov, Ilya and Nachum, Ofir and Tompson, Jonathan},
  booktitle={International Conference on Learning Representations},
  year={2020}
}

@inproceedings{luo2023otr,
  title={Optimal transport for offline imitation learning},
  author={Luo, Yicheng and Jiang, Zhengyao and Cohen, Samuel and Grefenstette, Edward and Deisenroth, Marc Peter},
  booktitle={International Conference on Learning Representations},
  year={2023}
}

@inproceedings{ma2022smodice,
  title={{SMODICE}: Offline imitation learning via stationary occupancy measure difference minimization},
  author={Ma, Yecheng Jason and Jayaraman, Dinesh and Bastani, Osbert},
  booktitle={International Conference on Machine Learning},
  year={2022}
}

@inproceedings{moulos2021bicausal,
  title={Bicausal optimal transport for {M}arkov chains via dynamic programming},
  author={Moulos, Vrettos},
  booktitle={IEEE International Symposium on Information Theory},
  year={2021}
}

@inproceedings{pham2025iostom,
  title={{IOSTOM}: Offline imitation learning from observations via state transition occupancy matching},
  author={Pham, Hung The and Doan, Tien Thanh and Nguyen, Thanh Thi and Phung, Dinh},
  booktitle={Advances in Neural Information Processing Systems},
  year={2025}
}

@book{puterman1994markov,
  title={Markov Decision Processes: Discrete Stochastic Dynamic Programming},
  author={Puterman, Martin L.},
  publisher={Wiley-Interscience},
  year={1994}
}

@inproceedings{sikchi2024dilo,
  title={A dual approach to imitation learning from observations with offline datasets},
  author={Sikchi, Harshit and Chuck, Caleb and Zhang, Amy and Niekum, Scott},
  booktitle={Conference on Robot Learning},
  year={2024}
}

@inproceedings{sun2019provably,
  title={Provably efficient imitation learning from observation alone},
  author={Sun, Wen and Vemula, Anirudh and Boots, Byron and Bagnell, Drew},
  booktitle={International Conference on Machine Learning},
  year={2019}
}

@inproceedings{torabi2018bco,
  title={Behavioral cloning from observation},
  author={Torabi, Faraz and Warnell, Garrett and Stone, Peter},
  booktitle={International Joint Conference on Artificial Intelligence},
  year={2018}
}

@inproceedings{vanbreugel2001algorithm,
  title={An algorithm for quantitative verification of probabilistic transition systems},
  author={van Breugel, Franck and Worrell, James},
  booktitle={International Conference on Concurrency Theory},
  year={2001}
}

@inproceedings{yan2024pwdice,
  title={Offline imitation from observation via primal {W}asserstein state occupancy matching},
  author={Yan, Kai and Schwing, Alexander G. and Wang, Yu-Xiong},
  booktitle={International Conference on Machine Learning},
  year={2024}
}

@inproceedings{zhang2021learning,
  title={Learning invariant representations for reinforcement learning without reconstruction},
  author={Zhang, Amy and McAllister, Rowan Thomas and Calandra, Roberto and Gal, Yarin and Levine, Sergey},
  booktitle={International Conference on Learning Representations},
  year={2021}
}

@inproceedings{sutton1999policy,
 author = {Sutton, Richard S and McAllester, David and Singh, Satinder and Mansour, Yishay},
 booktitle = {Advances in Neural Information Processing Systems},
 editor = {S. Solla and T. Leen and K. M\"{u}ller},
 pages = {},
 publisher = {MIT Press},
 title = {Policy Gradient Methods for Reinforcement Learning with Function Approximation},
 url = {https://proceedings.neurips.cc/paper_files/paper/1999/file/464d828b85b0bed98e80ade0a5c43b0f-Paper.pdf},
 volume = {12},
 year = {1999}
}

@book{danskin1967theory,
  author    = {Danskin, John M.},
  title     = {The Theory of Max-Min and its Application to Weapons Allocation Problems},
  year      = {1967},
  publisher = {Springer-Verlag},
  address   = {Berlin, Heidelberg}
}

\newpage
\appendix

\section{Proof of Theorem~\ref{thm:gradient}}
\label{app:proof}

We prove formula~\eqref{eq:gradient} in two steps.

\paragraph{Step 1: $\theta$ enters the Lagrangian only through $\nu_\theta$.}
Inspecting~\eqref{eq:lagrangian}, every term except the third line involves
only $\nu_X$, $\nu_0$, and the primal--dual variables $(\mu,\lambda_X,\lambda_Y,
\alpha_X,\alpha_Y,V)$; the sole $\theta$-dependent term is
\begin{equation}
  -\sum_{x,y,y'} \nu_\theta(y,y')\,\lambda_Y(x|y)\,\alpha_Y(x,y,y'),
  \label{eq:L_decomp}
\end{equation}
and the dependence is \emph{linear} in $\nu_\theta$.

\paragraph{Step 2: Apply the envelope theorem and differentiate.}
By the envelope theorem for parametric LPs~\citep{puterman1994markov}
\[
  \nabla_\theta\, d_\gamma(\mathcal{M}_X, \mathcal{M}_\theta)
  \;=\; \frac{\partial}{\partial\theta}\,
    \mathcal{L}\!\left(\mu^*, \lambda_X^*, \lambda_Y^*;\,
    \alpha_X^*, \alpha_Y^*, V^*\right)\bigg|_\theta,
\]
where $(\mu^*,\lambda_X^*,\lambda_Y^*;\alpha_X^*,\alpha_Y^*,V^*)$ is the
saddle point of~\eqref{eq:saddle}.
Since $\lambda_Y^*$ and $\alpha_Y^*$ are constants with respect to $\theta$
at the saddle point, differentiating~\eqref{eq:L_decomp} and exchanging the
(finite) sum with the derivative gives
\[
  \nabla_\theta\, d_\gamma(\mathcal{M}_X, \mathcal{M}_\theta)
  \;=\; -\sum_{x,y,y'} \lambda^*_Y(x|y)\,\alpha^*_Y(x,y,y')\,
    \nabla_\theta\,\nu_\theta(y,y'),
\]
which is exactly \eqref{eq:gradient}. \hfill$\square$

% ── Instantiation Algorithms ─────────────────────────────────────────────────

\section{Instantiation Algorithms}
\label{app:algorithms}

This appendix collects the full pseudocode for the two \textsc{D-BOT} instantiations described in the main text.
Both share the same outer structure as the general Algorithm~\ref{alg:dbot}: alternate between an
\textsc{SOMCOT} solve that produces dual certificates $(\bar\lambda_Y, \bar\alpha_Y)$ and a gradient
step that uses those certificates to update the model parameters.  What differs between them is
\emph{how} the gradient of the transition occupancy $\nabla_\theta \nu_\theta$ is computed, which
in turn reflects the different ways $\theta$ parametrizes the chain.

\paragraph{Algorithm~\ref{alg:compress} (\textsc{D-BOT-Repr}).}
This instantiation covers both parametric model learning and state-space compression
(Section~\ref{sec:representation}).  The transition kernel $P_\theta$ is an explicit
differentiable function of $\theta$, so $\nu_\theta$ is available in closed form via the linear
system~\eqref{eq:nu_theta}, and $\nabla_\theta \nu_\theta$ is obtained by differentiating through
the linear solve with standard autodiff.  The two sub-tasks (model learning with $n_Y = n_X$, and
compression with $n_Y \ll n_X$) share identical gradient machinery; the only difference is the
size of the abstract state space $\mathcal{Y}$.

\paragraph{Algorithm~\ref{alg:ilfo} (\textsc{D-BOT-ILfO}).}
This instantiation handles imitation learning from observations
(Section~\ref{sec:ilfo}), where $P_\theta$ is the kernel induced by a policy $\pi_\theta$ acting
in an MDP and is not directly differentiable.  Rather than differentiating through the dynamics,
the $\nu_\theta$-dependent term in the gradient is rewritten as a standard discounted RL objective
under an implicit reward $r_k$ derived from the dual certificates, and $\nabla_\theta \nu_\theta$
is then estimated via the policy gradient theorem. The policy optimization algorithm must run to 
convergence at each outer iteration, while the implicit reward is fixed. The implicit reward
$r_k$ then changes whenever $\pi_\theta$ is updated.

\begin{algorithm}[h]
  \caption{\textsc{D-BOT-Repr}: Representation Learning via Bicausal OT}
  \label{alg:compress}
  \begin{algorithmic}[1]
    \Require Samples from $\mathcal{M}_X$; model space size $n_Y \leq n_X$;
             cost $c$; discount $\gamma$; Iterations $K$;
             step size $\eta$.
             \State Initialize $\theta_0$;
    \For{$k = 1, \ldots, K$}
      \State $(\bar\lambda_{Y,k},\, \bar\alpha_{Y,k}) \leftarrow \textsc{SOMCOT}(\mathcal{M}_X,\, \mathcal{M}_{\theta_{k-1}},\, c)$ \hfill\textit{// OT step}
      \State \textbf{Gradient step:}
      \State \quad Compute $\nu_{\theta_{k-1}}$ via (Eq.~\ref{eq:nu_theta})
      \State \quad Compute $\widehat{\nabla_\theta \nu_{\theta_{k-1}}}$ via autodiff.
      \State \quad $\widehat{\nabla_\theta d_\gamma} \leftarrow -\bar\lambda_Y\;\bar\alpha_Y\;\widehat{\nabla_\theta \nu_{\theta_{k-1}}}.$
      \State \quad $\theta_k \leftarrow \theta_{k-1}
             - \eta\, \widehat{\nabla_\theta d_\gamma}$
    \EndFor \\
    \Return $P_{\theta_K}$;\quad
            encoder $\lambda_{X}$;\quad
            decoder $\bar\lambda_Y$.
  \end{algorithmic}
\end{algorithm}

\begin{algorithm}[h]
  \caption{\textsc{D-BOT-ILfO}: ILfO via Bicausal OT}
  \label{alg:ilfo}
  \begin{algorithmic}[1] 
    \Require Expert dataset $\mathcal{D}_E$; cost $c$; step size
             $\eta$; iterations $K$.
             \State Initialize $\pi_{\theta_0}$.
    \For{$k = 1, 2, \ldots, K$}
      \State \textbf{OT step (run until convergence):}
             \[
               (\bar\lambda_{Y,k},\, \bar\alpha_{Y,k})
               \leftarrow
               \textsc{SOMCOT}\!\left(\mathcal{M}_E,\,
               \mathcal{M}_{\pi_{\theta_{k-1}}},\, c\right)
             \]
      \State \textbf{Implicit reward (fixed):}
             \[
               r_k(s,s') = -\sum_a
               \bar\lambda_{Y,k}(a|s)\,\bar\alpha_{Y,k}(a,s,s')
             \]
      \State \textbf{Policy optimization (run until convergence):}
             update $\pi_{\theta_k}$ by any RL algorithm
             (e.g.\ policy gradient) maximising
             $\mathbb{E}_{\pi_\theta}\!\left[
               \sum_{t \geq 0} \gamma^t r_k(S_t, S_{t+1})
             \right]$
             with $r_k$ held fixed.
    \EndFor \\
    \Return $\pi_{\theta_{K}}$.
  \end{algorithmic}
\end{algorithm}

\section{Practical considerations}
\paragraph{Choice of ground cost.} The bisimulation metric depends on a user-specified ground cost (c(x,y)) that measures instantaneous mismatch between states. In the classical bisimulation metric literature for Markov decision processes, this cost is typically chosen as the absolute reward difference,
$c(x,y)=|r(x)-r(y)|$,
where $r$ is the state reward function \cite{ferns2004metrics}. Intuitively, two states are considered behaviorally similar if they yield similar immediate rewards and induce similar future transition structure.

More generally, $c$ may encode any task-relevant notion of local discrepancy between states, including feature-space distances or learned representation metrics. The choice of $c$ is therefore application dependent and constitutes part of the modelling assumptions of the method.

\paragraph{Softmax parametrization.}
\label{sec:softmax}
The gradient (Eq.~\ref{eq:grad_repr}) is parametrization-agnostic; any smooth map
$\theta \mapsto P_\theta$ can be plugged in.  For all experiments here we use
the logit matrix $\theta = \mathrm{vec}(W) \in \mathbb{R}^{n_Y \times n_Y}$:
\begin{equation}
  P_\theta(y'|y) \;=\; \frac{e^{W_{yy'}}}{\sum_{y''} e^{W_{yy''}}},
  \qquad \forall\, y, y' \in \mathcal{Y}.
  \label{eq:softmax}
\end{equation}
Softmax ensures row-stochasticity and $P_\theta(y'|y) > 0$ everywhere,
guaranteeing invertibility of $(I - \gamma P_\theta^\top)$.  
This implies that Softmax parametrization can't represent probabilities of exact mass 0.
Other kinds of parametrizations must be explored when required by the environment properties.

\paragraph{Computational cost of the linear solve.}
Solving (Eq.~\ref{eq:rho_system}) requires inverting the $n_Y \times n_Y$ matrix
$(I - \gamma P_\theta^\top)$, which costs $O(n_Y^3)$. Since \textsc{SOMCOT}
operates on the joint space $\mathcal{X} \times \mathcal{Y}$ and has cost
$O(n_X^2 n_Y^2)$, and $n_X \gg n_Y$ by assumption, the linear solve is not
a bottleneck in the tabular setting: $n_Y^3 \ll n_X^2 n_Y^2$.

\paragraph{Warm starting.}
Because consecutive iterates $\theta_k$ and $\theta_{k-1}$ differ by a small
step, the chain $\mathcal{M}_{\theta_k}$ changes slowly.  Re-initialising
\textsc{SOMCOT} from scratch at each outer step is therefore wasteful.  The
primal--dual variables from iteration $k-1$ provide a warm start for the
inner solve at iteration $k$, substantially reducing the number of inner
iterations needed.

\paragraph{Gradient bias.}
The envelope theorem is exact only at the true saddle point.  Since
\textsc{SOMCOT} terminates after finitely many inner iterations, the returned
duals are approximate, introducing a bias in the outer gradient proportional
to the inner approximation error.  Increasing $K_{\mathrm{in}}$ reduces this
bias at the cost of more computation per outer step.  This bias is analyzed
further in Section~\ref{sec:conclusion}.

% ── Ablation Studies ─────────────────────────────────────────────────────────

\section{Ablation Studies}
\label{app:ablations}

This appendix collects ablation experiments for both the compression and ILfO
instantiations of \textsc{D-BOT}.  The goal is to characterize sensitivity to
the main hyperparameters and to verify that the algorithm is robust across
reasonable choices.

\subsection{State-Space Compression Ablations}

\paragraph{Encoder structure.}
Figure~\ref{fig:ablation_encoder} shows the soft encoder $\lambda_X$ recovered
at convergence on the block-chain environment.  The encoder correctly maps
clusters of original states to single abstract states, confirming that the
coupling $\mu^*$ simultaneously identifies a meaningful aggregation without any
explicit clustering objective. 

\begin{figure}[h!]
  \centering
  \includegraphics[scale=0.6]{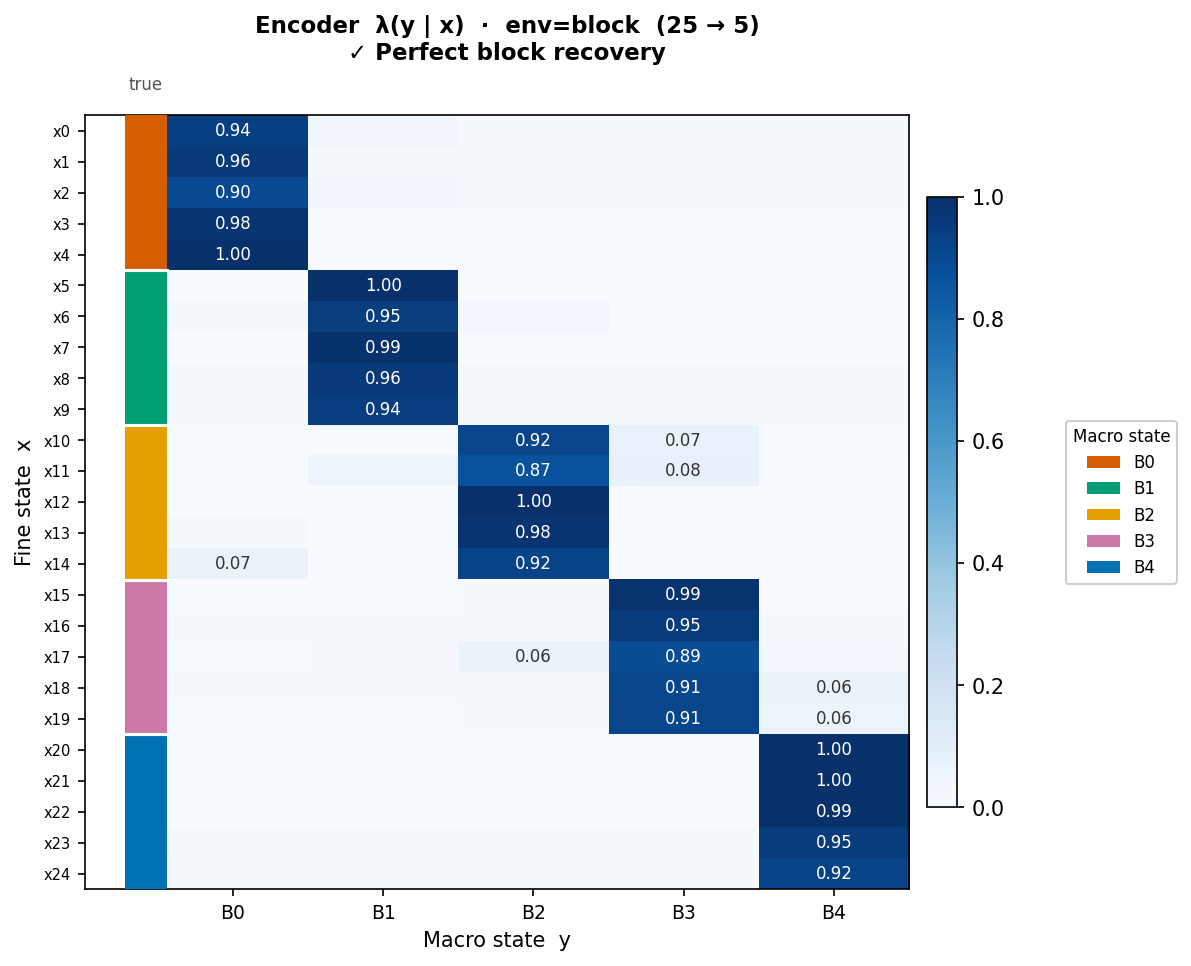}
  \caption{Recovered soft encoder $\lambda_X(y|x)$ at convergence on the
    block-chain environment ($n_X = 25$, $n_Y = 5$, $\gamma = 0.95$).  Rows
    correspond to original states $x \in \mathcal{X}$; columns to abstract
    states $y \in \mathcal{Y}$.  The encoder assigns each group of five
    original states to a distinct abstract state, recovering the true block
    structure.}
  \label{fig:ablation_encoder}
\end{figure}

\paragraph{Sensitivity to the outer learning rate $\eta_{\mathrm{out}}$.}
Figure~\ref{fig:ablation_eta} reports the final bisimulation distance as a
function of outer iterations for three values of the outer learning rate
$\eta_{\mathrm{out}}$. Too large a value causes oscillations in the outer
loop, while too small a value slows convergence without improving the final
solution. The final output is robust for the tested values. 

\begin{figure}[h!]
  \centering
  \includegraphics[scale=0.65]{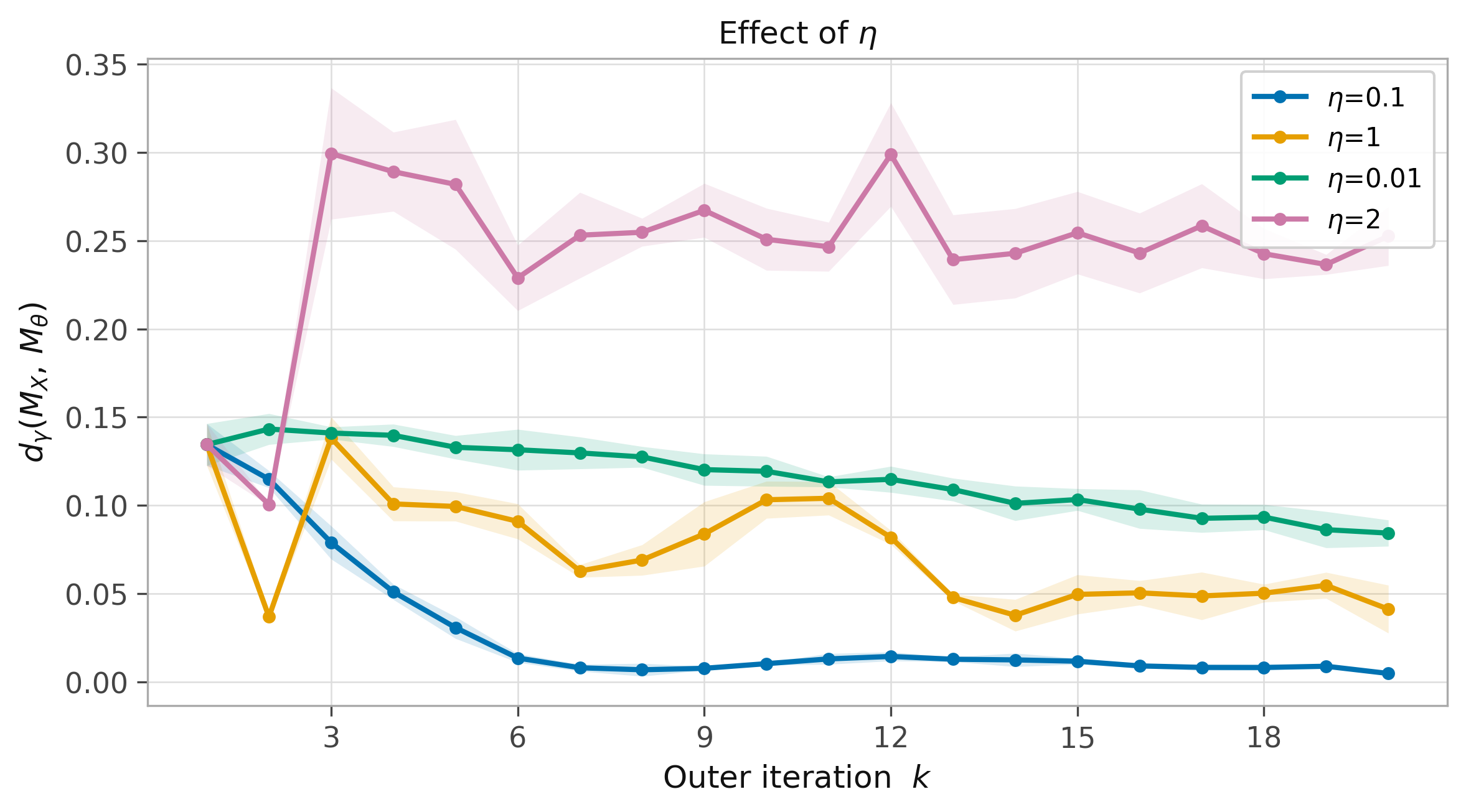}
  \caption{Effect of the outer learning rate $\eta_{\mathrm{out}}$ on the
    convergence of \textsc{D-BOT-Compress} (block-chain, $n_Y = 5$,
    $\gamma = 0.95$).  Moderate values converge reliably; an excessively
    large rate introduces oscillations.}
  \label{fig:ablation_eta}
\end{figure}

\paragraph{Sensitivity to the abstract space size $n_Y$.}
\label{app:model_size}
Figure~\ref{fig:ablation_model_size} addresses a practical question that
arises whenever \textsc{D-BOT-Compress} is deployed: the algorithm requires
$n_Y$ to be fixed before training, yet the true latent dimension of
$\mathcal{M}_X$ may not be known in advance.

The left panel shows the achieved bisimulation distance as a function of
$n_Y$.  When $n_Y$ is smaller than the true latent dimension (the
\emph{under-parametrized} regime), the abstract chain lacks the expressive
power to capture all the dynamics of $\mathcal{M}_X$, and the distance
remains large.  When $n_Y$ equals the true latent dimension, the distance
reaches its minimum.  Crucially, \emph{increasing $n_Y$ beyond the true
dimension does not degrade performance}: the algorithm simply leaves the
redundant abstract states unused, and the distance stays at its minimum.
This means that $n_Y$ need only be an upper bound on the true latent
dimension. Note that choosing a larger $n_Y$ does not hurt the quality of the learned model, but it does increase computational cost since \textsc{SOMCOT} scales as $\mathcal{O}(n_X^2 n_Y^2)$.

The right panel shows convergence curves for each value of $n_Y$ across outer
iterations, confirming that the over-parametrized runs ($n_Y > 5$) converge
to the same distance as the exactly-specified run ($n_Y = 5$), while the
under-parametrized runs ($n_Y < 5$) plateau at a higher distance.

\begin{figure}[h!]
  \centering
  \includegraphics[width=0.8\linewidth]{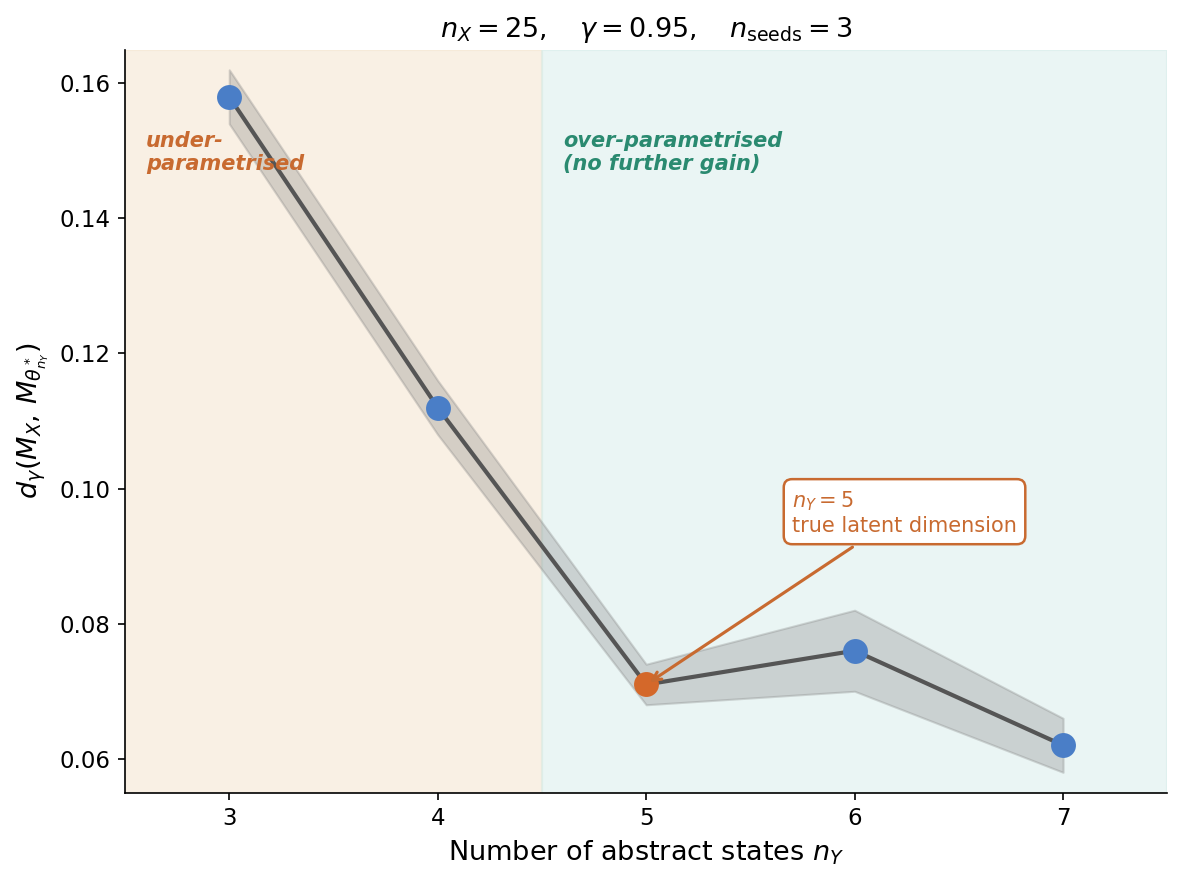}
  \caption{\textbf{Effect of the abstract space size $n_Y$ on compression
    quality} (block-chain, $n_X = 25 = 5 \times 5$, true latent dim $= 5$,
    $\gamma = 0.95$).  \emph{Left:} Final bisimulation distance
    $d_\gamma(\mathcal{M}_X, \mathcal{M}_{\theta^*})$ versus $n_Y$.  The
    distance drops sharply from the under-parametrized regime ($n_Y < 5$) to
    the true latent dimension ($n_Y = 5$, orange marker), then remains flat as
    $n_Y$ increases, this confirms that over-parametrization does not hurt the output.}
  \label{fig:ablation_model_size}
\end{figure}

% ── Extended Related Work ─────────────────────────────────────────────────────

\section{Extended Related Work}
\label{app:related_work}
 
\paragraph{Representation learning.}
Representation learning in reinforcement learning aims to find compact latent models that preserve the behavioral structure of the original process. Bisimulation metrics emerged as principled notions of behavioral equivalence for stochastic processes~\citep{desharnais1999metrics, ferns2004metrics,
vanbreugel2001algorithm}.
\citet{givan2003equivalence} showed that exact bisimulation induces state aggregations preserving optimal behavior, establishing a theoretical basis for state-space compression.
 
More recently, contributions have focused on scalable and differentiable approximations of these ideas.
\citet{castro2020scalable} proposed efficient algorithms for computing state similarity metrics in deterministic Markov decision processes (MDPs), enabling approximate aggregation in larger domains.
\citet{zhang2021learning} introduced invariant representation learning objectives for reinforcement learning, learning latent embeddings that preserve task-relevant behavioral structure. \citet{chen2022learning}
proposed robust behavioral metrics for deep reinforcement learning, learning representations that explicitly preserve transition dynamics under perturbations. \citet{kemertas2022approximate} incorporated bisimulation metrics into approximate policy iteration, demonstrating that behavior-aware metrics can improve both representation quality and control performance.
 
\paragraph{Imitation learning from observations.}
Different approaches to ILfO have been proposed.
Earlier work such as Behavioral Cloning from Observation (BCO) by \citet{torabi2018bco} learns an
inverse dynamics model to map state transitions back into actions before applying behavior cloning.
More closely related to this work, a large body of work formulates ILfO
as a problem of distribution matching: the learner attempts to match some informative distribution
(e.g.\ state occupancy measure) of the expert by minimizing a divergence or distance between them.
We can group the existing approaches into methods that minimize either the KL-divergence or the optimal transport distance. Within KL-based methods, \citet{kostrikov2019imitation}
proposed ValueDICE, which formulates imitation as stationary distribution matching through a Donsker-Varadhan
representation of the KL divergence. \citet{ma2022smodice} and \citet{kim2022lobsdice} propose variations
to the KL objective by adding regularization terms. More recently, \citet{pham2025iostom} introduced IOSTOM,
which matches state-transition occupancies without requiring an adversarial discriminator.
Within OT based methods, \citet{dadashi2021primal} introduced Primal Wasserstein Imitation
Learning (PWIL), which estimates a Wasserstein distance between state-action occupancy measures.
\citet{luo2023otr} proposed Optimal Transport for Offline Imitation Learning (OTR), computing
Wasserstein distances between empirical occupancy measures to relabel offline trajectories. \citet{sikchi2024dilo}
developed \textsc{DILO} (a dual formulation) for imitation from observation, combining state-only expert trajectories with offline RL data. \citet{chang2024oops} proposed an OT-based approach
for imitation from observation (OOPS) that defines trajectory-level distances. \citet{yan2024pwdice} jointly learn a ground metric via contrastive learning and propose Primal Wasserstein Imitation from Observations (PW-DICE), minimizing the Wasserstein distance between state occupancies. \citet{sun2019provably} study the sample complexity of ILfO in the
online setting, providing provable efficiency guarantees.
 
We argue that methods that match state-occupancy marginals are fundamentally limited: different
policies can yield identical marginal occupancies while producing qualitatively
different transition behaviors.  Divergence-based
methods~\citep{kostrikov2019imitation, ma2022smodice, kim2022lobsdice} also
suffer from sensitivity to distributional overlap, and static OT
methods~\citep{dadashi2021primal, luo2023otr, chang2024oops} compute couplings
over trajectory samples without respecting temporal causality.  Looking at the bisimulation metric (Eq.~\ref{eq:distance}) instead directly addresses these limitations, since it is defined
in the \emph{full causal-temporal structure} of the induced chain.

\end{document}